\pdfoutput=1

\documentclass[preprint,12pt,authoryear]{elsarticle}

\usepackage[utf8]{inputenc}
\usepackage{mathtools, amssymb}
\usepackage{enumitem}
\usepackage{xcolor}
\usepackage{kotex}
\usepackage{tabularray}
\usepackage{threeparttable}
\UseTblrLibrary{booktabs}
\usepackage{bbm}
\usepackage{multirow}

\newtheorem{condition}{Condition}

\newcommand{\narrow}[2]{{\mathcal{N}}^{#1}_{#2}}
\newcommand{\lr}[1]{\text{LR}_{#1}}

\newcommand{\BlackBox}{\rule{1.5ex}{1.5ex}}  
\newenvironment{proof}{\par\noindent{\bf Proof\ }}{\hfill\BlackBox\\[2mm]}
 
\newtheorem{theorem}{Theorem}
\newtheorem{lemma}[theorem]{Lemma} 
\newtheorem{proposition}[theorem]{Proposition} 
\newtheorem{remark}[theorem]{Remark}
\newtheorem{corollary}[theorem]{Corollary}
\newtheorem{definition}[theorem]{Definition}

\begin{document}

\begin{frontmatter}



\title{Minimum Width for Deep, Narrow MLP: A Diffeomorphism Approach}


\author{Geonho Hwang}

\affiliation{organization={Center for AI and Natural Sciences, Korea Institute for Advanced Study},
            addressline={85, Hoegi-ro}, 
            city={Dongdaemun-gu},
            postcode={02455}, 
            state={Seoul},
            country={Republic of Korea}}

\begin{abstract}
Recently, there has been a growing focus on determining the minimum width requirements for achieving the universal approximation property in deep, narrow Multi-Layer Perceptrons (MLPs).
Among these challenges, one particularly challenging task is approximating a continuous function under the uniform norm, as indicated by the significant disparity between its lower and upper bounds.
To address this problem, we propose a framework that simplifies finding the minimum width for deep, narrow MLPs into determining a purely geometrical function denoted as $w(d_x, d_y)$. 
This function relies solely on the input and output dimensions, represented as $d_x$ and $d_y$, respectively.
Two key steps support this framework.
First, we demonstrate that deep, narrow MLPs, when provided with a small additional width, can approximate a $C^2$-diffeomorphism.
Subsequently, using this result, we prove that $w(d_x, d_y)$ equates to the optimal minimum width required for deep, narrow MLPs to achieve universality.
By employing the aforementioned framework and the Whitney embedding theorem, we provide an upper bound for the minimum width, given by $\operatorname{max}(2d_x+1, d_y) + \alpha(\sigma)$, where $0 \leq \alpha(\sigma) \leq 2$ represents a constant depending on the activation function.
Furthermore, we provide a lower bound of $4$ for the minimum width in cases where the input and output dimensions are both equal to two.

\end{abstract}




\end{frontmatter}



\section{Introduction}
The \textit{universal approximation property} (UAP) refers to the capability of neural networks to approximate a wide range of functions.
As this property forms the foundation for the efficacy of neural networks, it has garnered significant interest within the research community.

Initial research focused mainly on two-layered Multilayer Perceptrons (MLPs). 
\citet{cybenko1989approximation} demonstrated that two-layered MLPs with sigmoidal activation functions possess the UAP for approximating continuous functions.
Later, \citet{leshno1993multilayer} expanded the scope of activation functions to more general ones.
In addition to two-layered MLPs, extensive investigation has been conducted into the UAP of \textit{deep, narrow MLPs}.
These MLPs have a constrained width and an arbitrary number of layers. 
Given the common use of MLPs with relatively modest widths and more than two layers in practical scenarios, the UAP of deep, narrow MLPs has attracted significant interest.

In this regard, a series of studies have been undertaken to determine the \textit{minimum width}, which is the necessary and sufficient width for the UAP. 
The minimum width depends on factors such as the input dimension $d_x$, the output dimension $d_y$, the activation function, and the type of norm employed.
For instance, \citet{lu2017expressive} demonstrated that deep, narrow MLPs with ReLU activation functions possess the UAP, leading to further research that narrowed down the minimum width range.
\citet{hanin2017approximating} extended the study to encompass arbitrary output dimensions $d_y$.
 \citet{johnson2018deep} showed that a width of $d_x$ is insufficient to achieve the UAP in continuous function spaces, while \citet{kidger2020universal} proved that a dimension of $d_x + d_y + 2$ is sufficient.
On the other hand, \citet{park2020minimum} presented the optimal minimum width for deep, narrow MLPs with ReLU activation functions in $L_p$ space.
Furthermore, \citet{cai2022achieve} explored the lower bound of the minimum width for arbitrary activation functions.

In this paper, we concentrate on the universal approximation of continuous functions under the uniform norm.
The previous results concerning uniform approximation are organized in Table \ref{table:result_table}.
So far, research on the minimum width for approximations under the uniform norm using continuous activation functions has suggested that the minimum width lies between $\operatorname{max}(d_x + 1, d_y)$ and $d_x + d_y$.
Recently, \citet{duan2023minimumleakyrelu} claimed that the upper bound could be reduced to $max(d_x + 1, d_y) + \boldsymbol{1}_{d_x+1 = d_y}$.
On the other hand, \citet{kim_park2023minimum} proved that the lower bound equals or exceeds $d_y+1$ if $d_y$ is less than or equal to $2d_x$.
This leads to the contradiction $ d_y+1 \leq w_{min} \leq d_y$ for $d_x +2 \leq d_y \leq 2d_x$.
Therefore, there should be a more rigorous proof of the minimum width for the uniform approximation of continuous functions.

In this context, we provide rigorous upper and lower bounds for the minimum width required for deep, narrow MLPs to possess the UAP.
It is substantiated by proving that the minimum width for deep, narrow MLPs with Leaky-ReLU activation function is equal to a geometrical function denoted as $w(d_x, d_y)$.
$w(d_x, d_y)$ is the required dimension of diffeomorphisms for approximating arbitrary continuous functions with $d_x$-dimensional input and $d_y$-dimensional output.
This is built upon the concept of the UAP of invertible neural networks.
Specifically, we employ the result of \citet{teshima2020coupling}, which demonstrated that approximating arbitrary $C^2$-diffeomorphisms is equivalent to approximating arbitrary single-coordinate transformations.
We prove that deep, narrow MLPs are capable of approximating single-coordinate transformations, thereby confirming their capability to approximate $C^2$-diffeomorphisms.
Using the above statement, we provide some upper and lower bounds.
By leveraging classical results from topological geometry, we establish that any continuous function can be approximated by MLPs with width $\operatorname{max}(2d_x+1,d_y)$.
Moreover, we provide the non-trivial lower bound $4$ for the case of input and output dimensions two, which provides a necessity of the framework.


Our contributions are as follows:
\begin{itemize}
    \item We prove that deep, narrow MLPs of width $d$ with Leaky-ReLU activation function can approximate any $C^2$-diffeomorphisms on $\mathbb{R}^d$.
    For more general activation functions, we prove that deep, narrow MLPs of width $d+1$ and $d+2$ with ReLU and general activation function, respectively, can approximate any $C^2$-diffeomorphisms on $\mathbb{R}^d$.
    \item We suggest the purely topological indicator $w(d_x, d_y)$, which is equal to the optimal minimum width for the UAP of deep, narrow MLP with Leaky-ReLU activation function. 
    \item Building on the above results, we prove that deep, narrow MLP with width $\operatorname{max}(2 d_x + 1, d_y) + \alpha(\sigma)$ can approximate any continuous function in $C(\mathbb{R}^{d_x}, \mathbb{R}^{d_y})$ on a compact domain, where $0\leq \alpha(\sigma)\leq 2$ is the constant depending on the activation function.
    \item We prove that width $4$ is the optimal minimum width that deep, narrow MLP to approximate arbitrary continuous function from a compact set $K\subset \mathbb{R}^2$ to $\mathbb{R}^2$.
\end{itemize}
\subsection{Organization}
The remainder of the paper is structured as follows:
In Section \ref{sec:notation_defintion}, we provide the notations and definitions that will be used.
Section \ref{sec:main_theories} addresses the key theorems in this paper.
    Subsection \ref{subsec:problem_formulation} formulated the core problem.
    Subsection \ref{subsec:approximating_diffeomorphism} explores the theorems and their corresponding proofs related to approximating a diffeomorphism using deep, narrow MLPs.
    Subsection \ref{subsec:diffeo_to_conti_embedding_extension} introduces a geometrical invariant that establishes a necessary and sufficient condition for the universal approximation of continuous functions.
    Subsection \ref{subsec:upper_bound} presents the proof of the universal approximation theorem for deep, narrow MLPs with a width of $\operatorname{max}(2d_x+1, d_y)$ and offers an alternative proof for the result by \citet{kidger2020universal}.
    Subsection \ref{subsec:lower_bound} establishes a minimum width lower bound of $4$ for the specific case when $d_x = d_y = 2$.
Section \ref{sec:conclusion} concludes the study.

\begin{table}[t]
\caption{A summary of known results on minimum width for universal approximation of continuous functions.
$K$ denotes a compact domain, and ''Conti.'' is short for continuous.}
\label{table:result_table}
\vspace{8pt}
\centering
\begin{threeparttable}
\begin{tabular}{|l |l|c|c  |  }
\hline 
\multicolumn{1}{|c|}{\textbf{Reference}} & \textbf{Domain} & \textbf{Activation $\sigma$} & \textbf{Upper\,/\,lower bounds}\\ 
\hline \hline
\cite{hanin2017approximating} & $C( K, \mathbb{R}^{d_y})$ & ReLU & $d_x + 1 \leq w_{min} \leq d_x + d_y$\\
\hline
\cite{johnson2018deep} & $C( K, \mathbb{R})$ & uniformly conti.\tnote{$\dag$} & $w_{min} \geq d_x + 1$\\
\hline
\multirow{2}{120pt}[-1pt]{\cite{kidger2020universal}} & $C( K, \mathbb{R}^{d_y})$ & conti.\ nonpoly\tnote{$\ddag$} &  $w_{min} \leq d_x + d_y + 1$\\
& $C( K, \mathbb{R}^{d_y})$ & nonaffine poly & $w_{min} \leq d_x + d_y + 2$\\
\hline 
\multirow{2}{120pt}[-1pt]{\cite{park2020minimum}} 
 &$C([0,1], \mathbb{R}^{2})$ &
ReLU & 
$w_{min} = 3 > \max \{d_x+1, d_y\}$\\
 &$C( K, \mathbb{R}^{d_y})$ &
ReLU+STEP & 
$w_{min} = \max \{d_x+1, d_y\}$\\
\hline
    \multirow{4}{120pt}[-1pt]{\cite{cai2022achieve}}
     &$C({K},\mathbb{R}^{d_y})$ &Arbitrary & $w_{\min} \ge \max(d_x,d_y)$  \\
     &$C({K},\mathbb{R}^{d_y})$& ReLU+FLOOR      & $w_{\min} = \max(d_x,d_y,2)$ \\
     &$C({K},\mathbb{R}^{d_y})$& UOE+FLOOR     & $w_{\min} = \max(d_x,d_y)$  \\
     &$C([0,1],\mathbb{R}^{d_y})$ & UOE    & $w_{\min} = d_y$  \\
    \hline
   \citet{kim_park2023minimum}& $C({K},\mathbb{R}^{d_y})$ & uniformly conti.\tnote{$\dag$}& $w_{\min} \geq d_y + \boldsymbol{1}_{ d_x<d_y \leq 2d_x}$ \\
    \hline
    \multirow{6}{120pt}[-1pt]{\textbf{Ours}}
     &$C({K},\mathbb{R}^{d_y})$ &Leaky-ReLU & $w_{\min} \leq \max(2d_x+1,d_y)$  \\
     &$C({K},\mathbb{R}^{d_y})$& ReLU    & $w_{\min} \leq \max(2d_x+1,d_y)+ 1$ \\
     &$C({K},\mathbb{R}^{d_y})$& conti.\ nonpoly\tnote{$\ddag$}     & $w_{\min} \leq \max(2d_x+1,d_y) + 2$  \\
     &$C([0,1]^2,\mathbb{R}^{2})$ & ReLU & $w_{\min} = 4 $  \\
     &$C([0,1]^2,\mathbb{R}^{2})$ & Leaky-ReLU & $w_{\min} = 4 $  \\
     &$C([0,1]^2,\mathbb{R}^{2})$ & uniformly conti.\tnote{$\dag$} & $w_{\min} \geq 4 $  \\
    \hline
\end{tabular}
\begin{tablenotes}
\item[$\dag$] requires that $\sigma$ is uniformly approximated by a sequence of one-to-one functions.
\item[$\ddag$] requires that $\sigma$ is continuously differentiable at at least one point (say $z$), with $\sigma'(z) \neq 0$.
\end{tablenotes}
\end{threeparttable}
\label{tbl:summary}
 \end{table}

\section{Notation and Definition}\label{sec:notation_defintion}
In this section, we introduce notations and definitions used throughout this paper.
\begin{itemize}
    \item $\mathbb{R}$ represents the set of real numbers.
    \item $\mathbb{R}_{+}$ denotes the set of positive real numbers.
    \item $\mathbb{N}$ is the set of natural numbers, and $\mathbb{N}_0 = \mathbb{N}\cup \{0\}$.
    \item For $a,b\in \mathbb{R}$, $[a,b]$ and $(a,b)$ represent the closed and open intervals from $a$ to $b$, respectively.
    \item $M_{n,m}$ is the set of $n\times m$ real matrices.
    \item $GL(n)\subset M_{n,n}$ is the set of invertible matrices.
    \item $\operatorname{Aff}_{n,m}$ and $\operatorname{IAff}_{n}$ are sets of affine transformations from $\mathbb{R}^n$ to $\mathbb{R}^m$ and invertible affine transformations from $\mathbb{R}^n$ to $\mathbb{R}^n$, respectively. 
    \item For a $d$-dimensional vector $x\in \mathbb{R}^d$, $x_i$ will denote the $i$-th component of $x$; that is, $x = (x_1,x_2,\dots, x_d)$. 
    And, $x_{i:j}$ will represent the $(j-i+1)$-dimensional vector $(x_i,x_{i+1},\dots, x_j)$.
\end{itemize}
\subsection{Compact Approximation}
$C(X,Y)$ represents the set of continuous functions from $X$ to $Y$. 
For a function $f\in C(X,Y)$ and a set $X'\subset X$, $f|_{X'}$ denotes a restriction of the function to the domain $X'$.
For a set of functions $\mathcal{A}\subset C(X,Y)$, $\mathcal{A}|_{X'}$ is defined as $\{f|_{X'}|f\in \mathcal{A}\}$.
We are concerned with the uniform approximation of a continuous function on a compact set, defined as follows:
\begin{definition}
    For two function spaces, $\mathcal{A}, \mathcal{B}\subset C(\mathbb{R}^n,\mathbb{R}^m)$, we say that $\mathcal{A}$ \textbf{compactly approximates} $\mathcal{B}$ if for any $f\in \mathcal{B}$, a compact set $K\subset \mathbb{R}^n$, and $\epsilon>0$, there exists $g\in \mathcal{A}$ such that 
\begin{equation}
    \|f-g\|_{\infty,K}:= \operatorname{sup}_{x\in K}\|f(x)-g(x)\|_{\infty}<\epsilon.
\end{equation}
This is denoted as $\mathcal{A} \succ \mathcal{B}$ or $\mathcal{B} \prec \mathcal{A}$.
\end{definition}
The compact approximation relation is transitive: if $\mathcal{A} \succ \mathcal{B}$, and $\mathcal{B} \succ \mathcal{C}$, then, $\mathcal{A} \succ \mathcal{C}$.
\begin{proof}
    Consider an arbitrary $f\in \mathcal{C}$ and a compact set $K\subset \mathbb{R}^m$.
    Because $\mathcal{B} \succ \mathcal{C}$, there exists $g\in \mathcal{B}$ such that $\|f-g\|_{\infty, K}<\frac{\epsilon}{2}$.
    Because $\mathcal{A} \succ \mathcal{B}$, there exists $h\in \mathcal{A}$ such that $\|g-h\|_{\infty, K}<\frac{\epsilon}{2}$.
    Then, $\|f-h\|_{\infty, K}<\|f-g\|_{\infty, K}+\|g-h\|_{ \infty, K} <{\epsilon}$, and the relation $\mathcal{A} \succ \mathcal{C}$ holds.
    \end{proof}
We also use the notation $f\prec \mathcal{A}$ to indicate that $\{f\}\prec \mathcal{A}$.
For a set of functions $\mathcal{A}\subset C(X,Y)$, $\overline{\mathcal{A}}$ is the closure with respect to the uniform norm.

\subsection{Activation Function}
We follow the commonly used condition for activation functions as proposed by \citet{kidger2020universal}.
\begin{condition}\label{condition:activation}
    An activation function $\sigma$ is a $C^1$-function near $\alpha\in \mathbb{R}$, with $\sigma'(\alpha)\neq 0$.
\end{condition}
We define several activation functions that satisfy Condition \ref{condition:activation}.
\begin{itemize}
    \item ReLU: $\text{ReLU}(x) := \begin{cases}
        x &\text{ if } x \geq 0\\ 0 &\text{ if }x<0
    \end{cases}$.
    \item Leaky-ReLU : $\lr{\beta}(x) := \begin{cases}
        x &\text{ if } x \geq 0\\ \beta x &\text{ if }x<0
    \end{cases}$.
\end{itemize}
Activation functions applied to vectors function as componentwise operators. 
For $x\in \mathbb{R}^d$,
    \begin{equation}
        \sigma(x) := (\sigma(x_1), \dots, \sigma_{x_d}).
    \end{equation}
    
\subsection{Deep, Narrow MLP}
A set of MLPs, denoted as $\narrow{\sigma}{d_0, d_1, \dots, d_N}$, is defined as follows: 
     \begin{equation}
         \narrow{\sigma}{d_0, d_1, \dots, d_N}:=\left\{\left. f: \mathbb{R}^{d_0}\rightarrow \mathbb{R}^{d_N} \right| W_i\in \operatorname{Aff}_{d_{i-1},d_{i}}, f(x)=W_N\circ \sigma\circ \dots \circ \sigma \circ W_1\right\}.
     \end{equation}
For Leaky-ReLU, an additional parameter $\beta$ can vary for each layer, resulting in the set $\narrow{\lr{}}{d_0, d_1, \dots, d_N}$:
     \begin{equation}
         \narrow{\lr{}}{d_0, d_1, \dots, d_N}:=\left\{\left. W_N\circ \lr{\beta_{N-1}}\circ \dots \circ \lr{\beta_{1}}\circ W_1: \substack{\mathbb{R}^{d_0}\\ \rightarrow \mathbb{R}^{d_N}}  \right| W_i\in \operatorname{Aff}_{d_{i-1},d_{i}}, \beta_i\in \mathbb{R}_{+}, \right\}.
     \end{equation}
     We can define a set of deep, narrow MLPs with input dimension $d_x$, output dimension $d_y$, and at most $n$ intermediate dimensions as follows:
     \begin{equation}
         \narrow{\sigma}{d_x, d_y, n}:= \bigcup_{N\in \mathbb{N}_0}\bigcup_{1\leq d_1, d_2, \dots, d_N \leq n}\narrow{\sigma}{d_x, d_1, d_2 \dots,  d_N,d_y} .
     \end{equation}
     For natural numbers $n\geq m\in \mathbb{N}$, we define the natural projection $p_{n,m}:\mathbb{R}^{n}\rightarrow \mathbb{R}^{m}$ and the inclusion $q_{m,n}$ as follows:
     \begin{equation}
         p_{n,m}: (x_1,\dots, x_n) \mapsto (x_1,\dots, x_m),
     \end{equation}
     and
     \begin{equation}
        q_{m,n}: (x_1,\dots, x_m)\mapsto (x_1, \dots, x_m, 0,\dots, 0).
    \end{equation}
    We can check that for any $f\in \narrow{\sigma}{d_x, d_y, n}$, $f$ can be decomposed as:
    \begin{equation}
        f = p_{n,d_y } \circ g\circ q_{d_x, n},
    \end{equation}
    where $g\in \narrow{\sigma}{n, n, n}$.

\subsection{Subsets of Diffeomorphisms}
In this section, we define several subsets of the set of diffeomorphisms.
\begin{definition}[Invertible Neural Networks]
For any natural number $d$, let $\mathcal{G}$ be a subset of invertible functions from $\mathbb{R}^d$ to $\mathbb{R}^d$. 
Then, $\operatorname{INN}_{\mathcal{G}}$ is defined as:
\begin{equation}
    \operatorname{INN}_{\mathcal{G}}: =\left\{\left. W_1\circ g_1\circ\dots \circ W_n\circ g_n \circ W_{n+1}\right| n\in \mathbb{N}, g_i\in\mathcal{G}, W_i\in \operatorname{IAff}_{d}\right\}
\end{equation}
\end{definition}
Note that the approximation capability of $\operatorname{INN}_{\mathcal{G}}$ remains unchanged even if $\mathrm{IAff}_d$, in the definition, is replaced with $\mathrm{Aff}_{d,d}$.

\begin{definition}[Diffeomorphism: $\mathcal{D}^{r}({U})$]
Let $U\subset \mathbb{R}^d$ be an open subset, and let $r$ be a non-negative integer or infinity.
     $\mathcal{D}^{r}({U})$ is the set of $C^r$-diffeomorphisms from $U$ to $\mathbb{R}^d$.
\end{definition}
\begin{definition}[Compactly supported diffeomorphism: $\operatorname{Diff}^r_c(\mathbb{R}^d)$]
    A diffeomorphism $f:\mathbb{R}^d\rightarrow \mathbb{R}^d$ is compactly supported if there exists a compact subset $K\subset \mathbb{R}^d$ such that for any $x\notin K$, $f(x)=x$.
    $\operatorname{Diff}^r_c(\mathbb{R}^d)$ is the set of all compactly supported $C^r$-diffeomorphisms from $\mathbb{R}^d$ to $\mathbb{R}^d$.
\end{definition}
\begin{definition}[Single-coordinate transformations: $\mathcal{S}^r_{c}(\mathbb{R}^d)$]
    $\mathcal{S}^r_{c}(\mathbb{R}^d)$ is the set of all compactly supported $C^r$-diffeomorphisms defined as follows:
    \begin{equation}
       \mathcal{S}^r_{c}(\mathbb{R}^d):= \left\{ \left.\tau\in \operatorname{Diff}^r_c(\mathbb{R}^d) \right| \tau(x)=(x_1,\dots, x_{d-1}, \tau_{d}(x)), \tau_d\in C(\mathbb{R}^d ,\mathbb{R})\right\}.
    \end{equation}
\end{definition}
\begin{definition}[Single-coordinate affine coupling flows]
$\text{ACF}_d$ is the set of all single-coordinate affine coupling flows defined as follows:
    \begin{equation}
      \text{ACF}_d:=  \left\{\left.(x_1,\dots, x_{d-1}, exp(s(x_{1:d-1})x_d  + t(x_{1:d-1})))\right| s,t\in C(\mathbb{R}^{d-1}, \mathbb{R})\right\},
    \end{equation}
\end{definition}

\section{Main Theorem}\label{sec:main_theories}
\subsection{Problem Formulation}\label{subsec:problem_formulation}
Our primary objective is to determine the minimum width $w_{\min}\in \mathbb{N}$ such that for any compact set $K\subset \mathbb{R}^n$, a continuous function $f\in C(K, \mathbb{R}^m)$ can be uniformly approximated by $\narrow{\sigma}{n,m, w_{\min}}$.
In other words, we want to find the value $w_{\min}(n,m,\sigma)$ such that
\begin{equation}
    w_{min}(n,m,\sigma):= \operatorname{min}\left\{l\in \mathbb{N}\left| C(\mathbb{R}^n, \mathbb{R}^m) \prec \narrow{\sigma}{n,m,l} \right.\right\}.
\end{equation}

\subsection{Approximating Diffeomorphisms}\label{subsec:approximating_diffeomorphism}
In this subsection, we initially establish the capability of deep, narrow MLPs to approximate diffeomorphisms.
\begin{theorem}\label{thm:nn_is_invertible_approximator}
Let $\sigma$ be a continuous function that satisfies Condition \ref{condition:activation}. Then, for a natural number $d\in \mathbb{N}$, the set $\narrow{\sigma}{d, d, d + \alpha(\sigma)}$ compactly approximates $\mathcal{D}^{2}(\mathbb{R}^d)$, where 
\begin{equation}
    \alpha(\sigma) = \begin{cases}
    0 &\text{ if } \sigma = \text{Leaky-ReLU} 
    \\ 1 & \text{ if } \sigma = \text{ReLU}
    \\ 2 & \text{ if } \sigma = \text{otherwise}
\end{cases}.
\end{equation}
 In other words, we have the relation
 \begin{equation}
     \mathcal{D}^{2}(\mathbb{R}^d) \prec \narrow{\sigma}{d,d, d + \alpha(\sigma)}.
 \end{equation}
\end{theorem}

To prove the theorem, we introduce a lemma that suggests we can focus on approximating $\mathcal{S}_{\mathrm{c}}^{\infty}(\mathbb{R}^d)$ to achieve the approximation of diffeomorphisms.
\begin{lemma}\label{lemma:equivialence}
The following relation holds:
\begin{equation}
     \mathrm{INN}_{\mathcal{S}_{\mathrm{c}}^{\infty}(\mathbb{R}^d)} \succ \mathcal{D}^{2}(\mathbb{R}^d).
\end{equation}

\end{lemma}
\begin{proof}
   This result is a direct consequence of Theorem 1(B) of \citet{teshima2020coupling}.
   Since $\mathcal{S}_{\mathrm{c}}^{\infty}(\mathbb{R}^d)$ is locally bounded, which is due to its continuity and invertible, it satisfies the conditions of Theorem 1. 
   Given that $\mathrm{INN}_{\mathcal{S}_{\mathrm{c}}^{\infty}(\mathbb{R}^d)} \succ \mathcal{S}_{\mathrm{c}}^{\infty}(\mathbb{R}^d)$, we can conclude that $\mathrm{INN}_{\mathcal{S}_{\mathrm{c}}^{\infty}(\mathbb{R}^d)} \succ \mathcal{D}^{2}(\mathbb{R}^d)$.
    
\end{proof}

\begin{proof}[Proof of Theorem \ref{thm:nn_is_invertible_approximator}]
    According to Lemma \ref{lemma:equivialence}, it suffices to prove that the set of neural networks can serve as an approximator for $\mathcal{S}^{\infty}_{c}$: $\narrow{\sigma}{d, d, d + \alpha(\sigma)}\succ \mathcal{S}^{\infty}_{c}(\mathbb{R}^d) $.
    
    For $\sigma = \text{Leaky-ReLU}$, we need to prove that $\narrow{\sigma}{d, d, d}\succ \mathcal{S}^{\infty}_{c}(\mathbb{R}^d) $.
       We can accomplish this by employing Lemma \ref{lemma:approximating_single_coordinate_transforms}.
   
    In the case of $\sigma = \text{ReLU}$, by Theorem 1 in \cite{hanin2017approximating}, for $f(x) = (x_1, \dots, x_d, \tau(x))$, we have $f\prec {\narrow{\sigma}{d, d+1, d+1}}$. 
    Therefore, $\left(x_1,\dots, x_{d-1}, \tau(x)\right) \in \narrow{\sigma}{d, d, d+1}$, implying that $ \mathcal{S}^{\infty}_{c}(\mathbb{R}^d)\prec \narrow{\sigma}{d, d, d+1}$.

   For other continuous activation functions $\sigma$, Proposition 4.2 of \cite{kidger2020universal} demonstrates that for $f(x) = (x_1, \dots, x_d, \tau(x))$, we have $f\prec {\narrow{\sigma}{d, d+1, d+2}}$. 
    Consequently, $\left(x_1,\dots, x_{d-1}, \tau(x)\right) \in \narrow{\sigma}{d, d, d+2}$, concluding that $ \mathcal{S}^{\infty}_{c}(\mathbb{R}^d)\prec \narrow{\sigma}{d, d, d+2}$.
    
    It is important to note that although the original statements in Theorem 1 of \cite{hanin2017approximating} and Proposition 4.2 of \cite{kidger2020universal} do not explicitly state that the approximated function has the form $(x_1,\dots, x_d, \tau(x))$, their proofs implicitly make use of this form.
\end{proof}
Now, the remaining task is to prove the following lemma for the Leaky-ReLU case.
\begin{lemma}[Single-Coordinate Transformations to Leaky-ReLU]\label{lemma:approximating_single_coordinate_transforms}
The following relation holds:
\begin{equation}
    \narrow{\lr{}}{d, d, d}\succ \mathcal{S}^{\infty}_{c}(\mathbb{R}^d) .
\end{equation}
\end{lemma}
The proof of this lemma involves a series of lemmas and corollaries that gradually extend the scope of functions that can be approximated using Leaky-ReLU.

The following lemma implies that any increasing function can be approximated by composing Leaky-ReLUs and affine transformations.
    \begin{lemma}[Increasing Functions to Leaky-ReLU]\label{lemma:increasing_with_leaky}
     Define the sets as follows:
     \begin{equation}
         U_0 := \left\{\left. ax +b\right| a\in \mathbb{R}_{+}, b\in \mathbb{R}\right\},
     \end{equation}
\begin{equation}
    U_{n+1}:= \left\{ \left. a\lr{\beta} (f )+b\right|a, \beta\in \mathbb{R}_{+}, b\in \mathbb{R}, f\in U_n\right\},
\end{equation}
\begin{equation}
    U:= \bigcup_{n=0}^{\infty} U_{n}.
\end{equation}
    Then, for any continuous, increasing activation function $\sigma:\mathbb{R}\rightarrow \mathbb{R}$, 
    \begin{equation}
     \sigma\prec U.   
    \end{equation}
    \end{lemma}
The proof of Lemma \ref{lemma:increasing_with_leaky} is provided in Appendix \ref{Proof of Lemma_increasing_with_leaky}.
The lemma directly implies the subsequent corollary: deep, narrow MLPs with Leaky-ReLU activation function can approximate a deep, narrow MLP with an increasing activation function and the same width.
\begin{corollary}[Generalization of Activation]
     For any continuous, increasing activation function $\sigma$, the following relation holds: 
    \begin{equation}\label{eq:activation_generalization}
        \narrow{\sigma}{d, d, d}\prec \narrow{\lr{}}{d, d, d}.
    \end{equation}
\end{corollary}   
Using the above corollary, we can prove that any ACF can be approximated by Leaky-ReLU deep, narrow MLPs. 
    \begin{lemma}[ACF to Leaky-ReLU]\label{lemma:approximating_acf}
    The following relation holds:
    \begin{equation}
        \text{INN}_{\text{ACF}_d}\prec\narrow{\lr{}}{d, d, d}.
    \end{equation}
        \end{lemma}
        Proof of Lemma \ref{lemma:approximating_acf} is provided in Appendix \ref{proof_of_lemma_approximating_acf}.
Next, we establish a technical lemma.
This lemma serves as the multidimensional counterpart of Lemma \ref{lemma:increasing_with_leaky}.
For a multidimensional function from $\mathbb{R}^d$ to $\mathbb{R}$ increasing with a coordinate $x_d$, we can freely change the value when $x_d $ is large, while the value remains unaffected when $x_d$ is small. 
    \begin{lemma}\label{lemma:leaky_ReLU_technical}
        Consider a compact set $K = [0,1]^d \subset\mathbb{R}^{d}$, two distinct real values $\alpha_1 < \alpha_2$, and a single-coordinate transformation $F = (x_1,\dots,x_{d-1}, f(x))\in \mathcal{S}^r_c$.
        The function $f(x)$ satisfies the following relation:
        \begin{equation}
            f(x)\leq 0 \text{ if } x_d < \alpha_1, \text{ and }  f(x)= 0 \text{ if } x_d = \alpha_1.
        \end{equation}
        Assuming that $F\in \overline{\left.\narrow{\lr{}}{d,d,d}\right|_{K}}$, for a continuous function $b:\mathbb{R}^{d-1}\rightarrow\mathbb{R}$ such that $b(x_{1:d-1}) >0 $ for all $x\in K$, there exists a single-coordinate transformation $G = (x_1, \dots, x_{d-1},g(x))\in \overline{\left.\narrow{\lr{}}{d,d,d}\right|_{K}}$ that satisfies the following relation:
        \begin{equation}
            g(x):= \begin{cases}
                f(x) &\text{ if } x_d\leq \alpha_1
              \\  f(x)b(x_{1:d-1})  &\text{ if } x_d = \alpha_2
            \end{cases}.
        \end{equation}
    \end{lemma}
The proof of Lemma \ref{lemma:leaky_ReLU_technical} is provided in Appendix \ref{proof_leaky_ReLU_technial}.
With the help of this lemma, we can prove Lemma \ref{lemma:approximating_single_coordinate_transforms}.
\begin{proof}[Proof of Lemma \ref{lemma:approximating_single_coordinate_transforms}]
    Consider an arbitrary single-coordinate transformation $F(x) = \left(x_1,\dots, x_{d-1}, \tau(x_1,\dots, x_d) \right)$ and a compact set $K\subset \mathbb{R}^d$.
    Without loss of generality, we assume that $K = [0,1]^d$. In the case where $K$ is not $[0,1]^d$, we can rescale the domain of the function to fit within $[0,1]^d$ and continuously extend the domain to $[0,1]^d$.
    Additionally, assume that $\tau$ is strictly increasing with respect to $x_d$.
    
    Because $\tau$ is a continuous function defined on a compact set, for an arbitrary $\epsilon >0$, a natural number $N\in \mathbb{N}$ exists such that if $\|x - x'\|<\frac{1}{N}$, then $|\tau(x) - \tau(x')|<\epsilon$.
    Now, define $u_i:\mathbb{R}^{d-1}\rightarrow \mathbb{R}$ as follows:
    \begin{equation}
        u_i(x_{1:d-1}):= F\left(x_{1:d-1}, \frac{i}{N}\right).
    \end{equation}
    If there exists a single-coordinate transformation $G = (x_1, \dots, x_{d-1}, g(x_{1:d}))\prec \narrow{\lr{}}{d,d,d}$ such that $u_i(x_{1:d-1})=g\left(x_{1:d-1},\frac{i}{N}\right)$ for $x\in K$, then $F\prec  \narrow{\lr{}}{d,d,d}$.
    We will demonstrate the existence of a sequence $\left\{G_n\right\}^{\infty}_{n=1}\subset  \overline{\left.\narrow{\lr{}}{d,d,d}\right|_{K}}$ such that $G_n(x_{1:d-1}, \frac{i}{N}) = u_{i}(x_{1:d-1})$ for $1\leq i\leq n$ through mathematical induction.
    
    By Lemma \ref{lemma:leaky_ReLU_technical}, there exists a single-coordinate transformation $G_0 = (x_1,\dots, x_{d-1}, g_0(x))$ such that $g_0(x_{1:d-1},0) = u_0(x_{1:d-1})$.
    Assume that the induction hypothesis holds for some $n=n_0$, meaning that there exists a single-coordinate transformation $G_{n_0} = (x_1,\dots, x_{d-1}, g_{n_0}(x))$ such that $g_{n_0}(x_{1:d-1},\frac{i}{N}) = u_i(x_{1:d-1})$ for $1\leq i\leq n_0$.
    Then, by Lemma \ref{lemma:approximating_acf}, we can construct $G_{n_0}':=(x_1, \dots, x_{d-1}, g_{n_0}(x) - u_{n_0}(x_{1:d-1}))\in \overline{\left.\narrow{\lr{}}{d,d,d}\right|_{K}}$.
    Notably, $G'_{n_0}$ satisfies the assumptions of Lemma \ref{lemma:leaky_ReLU_technical} with $\alpha_1= \frac{n_0}{N}$ and $\alpha_2= \frac{n_0+1}{N}$.
    By applying Lemma \ref{lemma:leaky_ReLU_technical} with $b(x_{1:d-1}) = \frac{u_{n_0+1}(x_{1:d-1})- u_{n_0}(x_{1:d-1})}{ g_{n_0}(x_{1:d-1}, \frac{n_0+1}{N})- u_{n_0}(x_{1:d-1})}$, we obtain a single-coordinate transformation $G''(n_0)= (x_1,\dots, x_{d-1}, g''_{n_0}(x))$ such that $g''_{n_0}(x_{1:d-1}, \frac{i}{N}) = u_{i}(x_{1:d-1})-u_{n_0}(x_{1:d-1})$ for $i\leq n_0+1$.
    Finally, by Lemma \ref{lemma:approximating_acf}, we can get $G_{n_0+1}:= (x_1, \dots, x_{d-1}, g''_{n_0}(x) + u_{n_0}(x_{1:d-1}))\in \overline{\left.\narrow{\lr{}}{d,d,d}\right|_{K}}$.
    As a result, the induction hypothesis is satisfied, and this completes the proof.
\end{proof}

\subsection{Diffeormorphism to Continuous Function}\label{subsec:diffeo_to_conti_embedding_extension}
In this subsection, we aim to prove that any continuous function can be approximated by composing linear transformations and diffeomorphisms and to determine the required width $w(n,m)$ for approximation.  
Moreover, we will prove that the network-independently defined value $w(n,m)$ equals the minimum width of deep, narrow, Leaky-ReLU MLPs. 

Let $\mathrm{Emb}(X, Y)$ be the set of smooth embeddings from $X$ to $Y$.
Let $\mathrm{Emb}_{p.l.}(X, Y)$ be the set of smooth embeddings from $X$ to $Y$.
For natural numbers $d_1 \geq d_2$, let $p_{d_1,d_2}:\mathbb{R}^{d_1}\rightarrow \mathbb{R}^{d_2}$ be a projection to the first $d_2$ coordinates.
Define $w({n,m})$ as
\begin{equation}
    w({n,m}):=\operatorname{min}\left\{l\in \mathbb{N}_0\left| p_{l, m} \left( \overline{\mathrm{Emb}([0,1]^n, \mathbb{R}^{l})}\right) = C([0,1]^n, \mathbb{R}^m) \right.\right\}.
\end{equation}
Intuitively, $w(n,m)$ is the least width required to approximate an arbitrary continuous function with diffeomorphism.
\begin{remark}
    We remark that interval $[0,1]$ can be replaced with interval $[a, b]$ for $a <b$.
    And $\operatorname{Emb([0,1]^n, \mathbb{R}^l)}$ can be replaced with any dense subset of $ \overline{\mathrm{Emb}([0,1]^n, \mathbb{R}^{l})}$, such as $\mathrm{Emb}_{p.l.}([0,1]^n, \mathbb{R}^{l})$, the set of piecewise linear embedding from $[0,1]^n$ to $\mathbb{R}^{l}$ \citep{munkres1960obstructions_pl_to_diff}.
\end{remark}

We will prove that $w(n,m)$ has a similar value to $w(n,m, \sigma)$ and the same value to $w(n,m, \text{Leaky-ReLU})$. 
The next lemma demonstrates that any smooth embedding can be represented by composition of inclusion and smooth diffeomorphism.
\begin{lemma}[Theorem C of \cite{palais1960extending}]\label{lemma:diffeomorphism_extension}
     Consider natural numbers $n$ and $m$ where $n\leq m$,
     and an embedding $f:K = [0,1]^n\rightarrow\mathbb{R}^m$.
     Then, there exists a smooth diffeomorphism $F:\mathbb{R}^m\rightarrow\mathbb{R}^m$ such that the following equation holds:
     \begin{equation}
         F \circ q_{n,m} =f. 
     \end{equation}
\end{lemma}

Then, the following theorem is satisfied.
        \begin{theorem}\label{thm:w_approximable}
Let $\sigma$ be a continuous function satisfying Condition \ref{condition:activation}.
Then, $\narrow{\sigma}{n, m, w({n,m}) + \alpha(\sigma)}$ compactly approximates $C(\mathbb{R}^n, \mathbb{R}^m)$, where 
\begin{equation}
    \alpha(\sigma) = \begin{cases}
    0 &\text{ if } \sigma = \text{Leaky-ReLU} 
    \\ 1 & \text{ if } \sigma = \text{ReLU}
    \\ 2 & \text{ if } \sigma = \text{otherwise}
\end{cases}.
\end{equation}
 In other words, 
 \begin{equation}
     C(\mathbb{R}^n,\mathbb{R}^m) \prec \narrow{\sigma}{n,m, w({n,m}) + \alpha(\sigma)}.
 \end{equation}
\end{theorem}
\begin{proof}
Without loss of generality, assume that $K = [0,1]^n$. 
In other cases, we can continuously extend the function to a cube containing $K$ and rescale.
By the definition of $w(n,m)$, for arbitrary $f\in C([0,1]^n, \mathbb{R}^m)$ and $\epsilon>0$, there exists an embedding $g\in \operatorname{Emb}([0,1]^n,\mathbb{R}^{w(n,m)})$ such that 
\begin{equation}
    \|f - p_{w({n,m}), n}\circ g\|_{ \infty,[0,1]^n}<\epsilon.
\end{equation}
Because $w(n,m)\geq n$, by Lemma \ref{lemma:diffeomorphism_extension}, for $q_{n,w(n,m)}: (x_1,\dots, x_n)\mapsto (x_1, \dots, x_n, 0,\dots, 0)$, there exists a smooth diffeomorphism $G$ such that $g = G \circ q_{n,w(n,m)}$.
    By Theorem \ref{thm:nn_is_invertible_approximator}, there exists an $H\in \narrow{\sigma}{w(n,m),w(n,m),w(n,m)+\alpha(\sigma)}$ such that 
    \begin{equation}
        \|G-H\|_{\infty, K\times [0,1]^{w(n,m)-n}} <\epsilon.
    \end{equation}
    Then, 
    \begin{equation}
        \|p_{w({n,m}), n}\circ H \circ q_{n,w(n,m)} - p_{w({n,m}), n}\circ G \circ q_{n,w(n,m)}\|_{\infty, [0,1]^n}<\epsilon.
    \end{equation}
    Therefore,
    \begin{equation}
        \|f - p_{w(n,m),m}\circ H \circ q_{n,w(n,m)}\|_{\infty, K}<2\epsilon. 
    \end{equation}
     $p_{w(n,m),m}\circ H \circ q_{n,w(n,m)}\in \narrow{\sigma}{w(n,m),w(n,m),w(n,m)+\alpha(\sigma)}$, and it complete the proof.
\end{proof}

Furthermore, we can give the lower bound of the minimum width for the universal approximation.

\begin{proposition}\label{prop:w_not_approximable}
Let $\sigma$ be an increasing, continuous activation function.
    For $n,m\in \mathbb{N}$, $\narrow{\sigma}{n, m, w(n,m) -1}$ does not compactly approximate $C(\mathbb{R}^n, \mathbb{R}^m)$. 
 In other words, the following relation holds:
 \begin{equation}
     C(\mathbb{R}^n, \mathbb{R}^m) \nprec \narrow{\sigma}{n,m,  w(n,m) -1}.
 \end{equation}
\end{proposition}
\begin{proof}
    For an increasing continuous activation function $\sigma$, there exist smooth, strictly increasing activation functions $\sigma_n$ such that uniformly converge to $\sigma$.
    Therefore, $\narrow{\sigma}{d,d,d} \prec \bigcup_{n\in \mathbb{N}} \narrow{\sigma_n}{d,d,d} \prec\mathcal{D}^{\infty}(\mathbb{R}^d)$, and it is sufficient to consider only smooth, strictly increasing activation function $\sigma$.

    For $f\in \narrow{\sigma}{n,m,  w(n,m) -1}$, $f$ can be decomposed as
    \begin{equation}
        f = p_{w(n,m)-1,m} \circ g \circ q_{n, w(n,m)-1},
    \end{equation}
    where $g\in \narrow{\sigma}{w(n,m)-1,w(n,m)-1,w(n,m)-1}$.
    Because $\narrow{\sigma}{w(n,m)-1,w(n,m)-1,w(n,m)-1} \prec \mathcal{D}^{\infty}(\mathbb{R}^{w(n,m)-1})$, $\left. g \circ q_{n, w(n,m)-1}\right|_{[0,1]^n}\in \overline{\mathrm{Emb}([0,1]^n, \mathbb{R}^{w(n,m) -1})}$.
    Therefore, 
    \begin{equation}
        f|_{[0,1]^n}\in p_{w(n,m) -1, m} \left( \overline{\mathrm{Emb}([0,1]^n, \mathbb{R}^{w(n,m) -1})}\right),
    \end{equation}
   and because $f\in \narrow{\sigma}{n,m,  w(n,m) -1}$ is arbitrary, we have
   \begin{equation}
       \left.\narrow{\sigma}{n,m,  w(n,m) -1}\right|_{[0,1]^n}\subset p_{w(n,m) -1, m}\left( \overline{\mathrm{Emb}([0,1]^n, \mathbb{R}^{w(n,m) -1})}\right).
   \end{equation}
    Because $w(n,m) -1 < w(n,m)$, by the definition of $w(n,m)$, 
    \begin{equation}
        p_{w(n,m) -1, m} \left( \overline{\mathrm{Emb}([0,1]^n, \mathbb{R}^{w(n,m) -1})}\right) \nsupseteq C([0,1]^n, \mathbb{R}^m),
    \end{equation}
    and 
    \begin{equation}
        \narrow{\sigma}{n,m,  w(n,m) -1} \nsupseteq C([0,1]^n, \mathbb{R}^m).
    \end{equation}
    Thus, we can conclude that $C(\mathbb{R}^n, \mathbb{R}^m) \nprec \narrow{\sigma}{n,m,  w(n,m) -1}$.
\end{proof}
 By combining Theorem \ref{thm:w_approximable} and Proposition \ref{prop:w_not_approximable} together, we can conclude that the minimum width $w_{\operatorname{min}}(n,m,\text{Leaky-ReLU})$ is equal to $w(n,m)$ for Leaky-ReLU and can get tight inequality for general increasing activation functions.
\begin{corollary}\label{cor:lower_upper}
  The following equation holds:
   \begin{equation}
       w_{\operatorname{min}}(n,m,\text{Leaky-ReLU}) = w(n,m)
   \end{equation}
    For a general increasing activation function $\sigma$, which satisfies Condition \ref{condition:activation}, the following inequality holds:
   \begin{equation}
       w(n,m)\leq w_{\operatorname{min}}(n,m,\sigma)\leq w(n,m) + \alpha(\sigma),
   \end{equation}
where 
\begin{equation}
    \alpha(\sigma) = \begin{cases}
     1 & \text{ if } \sigma = \text{ReLU}
    \\ 2 & \text{ if } \sigma = \text{otherwise}
\end{cases}.
\end{equation}
\end{corollary}

\subsection{Some Observation about Upper bound of $w(n,m)$}\label{subsec:upper_bound}
In the previous subsection, we demonstrated that the minimum width of the deep, narrow MLP is fundamentally correlated with $w(n,m)$.
In this subsection, we will present the sufficient condition for the $w(n,m)$ to be equal to $m$.  
The following lemma demonstrates that a continuous function can be approximated by a smooth embedding when the output dimension is larger than twice the input dimension:
\begin{lemma}\label{lemma:whitney_lemma}
 Consider natural numbers $n$ and $m$ where $m>2n$.
    Let $f: K = [0,1]^n\subset \mathbb{R}^n \rightarrow \mathbb{R}^m$ be a continuous function.
    Then, for $\epsilon \in \mathbb{R}_{+}$, there exists a smooth embedding $g:K\rightarrow\mathbb{R}^m$ such that
    \begin{equation}
        \|f-g\|_{\infty, K}<\epsilon.
    \end{equation}
\end{lemma}
\begin{proof}
    Consider a connected, open subset $U$ of $\mathbb{R}^n$ such that $K\subset U\subset\mathbb{R}^n$.
    Since $K$ is compact, there exists a continuous extension $f_0$ of $f$ such that 
    \begin{equation}
        f_0|_{K} = f.
    \end{equation}
    As $U$ is a manifold, the assumptions of Theorem 3.17 and 3.18 of \cite{persson2014whitney} are satisfied.
    Therefore, there exists an injective immersion $g$ such that 
    \begin{equation}
        \|f-g\|_{\infty, U} < \epsilon.
    \end{equation}
    Consequently,  the restriction $g|_{K}$ defined on the compact set $K$ becomes a smooth embedding.
\end{proof}

\begin{theorem}\label{thm:main_thm}
Let $\sigma$ be a continuous function that satisfies Condition \ref{condition:activation}.
Then, for any natural numbers $n,m\in \mathbb{N}$, the set $\narrow{\sigma}{n, m, \operatorname{max}(2n+1, m) + \alpha(\sigma)}$ compactly approximates $C(\mathbb{R}^n, \mathbb{R}^m)$, where 
\begin{equation}
    \alpha(\sigma) = \begin{cases}
    0 &\text{ if } \sigma = \text{Leaky-ReLU} 
    \\ 1 & \text{ if } \sigma = \text{ReLU}
    \\ 2 & \text{ if } \sigma = \text{otherwise}
\end{cases}.
\end{equation}
 In other words, we have the relation
 \begin{equation}
     C(\mathbb{R}^n,\mathbb{R}^m) \prec \narrow{\sigma}{n,m, \operatorname{max}(2n+1, m) + \alpha(\sigma)}.
 \end{equation}
\end{theorem}
\begin{proof}
    Lemma \ref{lemma:whitney_lemma} implies that $w(n,m)\leq \operatorname{max}(2n+1, m)$. 
By Theorem \ref{thm:w_approximable}, we can immediately get the conclusion. 
\end{proof}
\begin{remark}
    As previously mentioned in the introduction, \citet{kim_park2023minimum} demonstrated that the minimum width $w_{\min}(d_x,d_y, \sigma)$ satisfies the relation $w_{\min} \geq d_y + \boldsymbol{1}_{ d_x<d_y \leq 2d_x}$ for an increasing activation function.
    It indicates that when the output dimension $d_y$ is twice the input dimension $d_x$, and the activation function is Leaky-ReLU, $w_{\min}(d_x, 2d_x, \sigma)$ is equal to or larger than $d_y+1$. 
   In the same configuration, according to Theorem \ref{thm:main_thm}, we can get the following relation: $w_{\min} \leq 2d_x+1 = d_y+1$.
     By combining these two results, we arrive at the optimal minimum width $w_{\min} = d_y+1 = 2d_x +1$.
\end{remark}

Besides the above relation, there is an obvious upper bound for $w(n,m)$:
\begin{equation}
    w(n,m)\leq n+m,
\end{equation}
for all $n,m\in \mathbb{N}$.
It reproofs the result of \citet{hanin2017approximating} for the Leaky-ReLU case:
\begin{equation}
    w_{min}(n,m ,\text{Leaky-ReLU})\leq n+m,
\end{equation}
and slightly worse results for ReLU \citep{hanin2017approximating} and other general activation functions \citet{kidger2020universal}:
\begin{equation}
     w_{min}(n,m ,\text{Leaky-ReLU})\leq n+m + \alpha(\sigma),
\end{equation}
where $\alpha(\sigma) =1$ for $\sigma = \text{ReLU}$, and $\alpha(\sigma) =2$ for other activation functions.


\subsection{Lower Bound of $w(n,m)$}\label{subsec:lower_bound}
In this subsection, we provide a nontrivial example of minimum width using the concept of $w(n,m)$.
In particular, we will prove that $w(2,2) = 4$ using some algebraic topological techniques.

We use the following lemma, which implies that the homology of the level set of a function is robust to the perturbation. 
\begin{lemma}[Theorem 2 of \citet{bendich2010robustnesslevelset}]\label{lemma:robustlevelset}
Let $\mathbb{X}$ be a compact topological space. For a continuous function $f:\mathbb{X}\rightarrow \mathbb{R}$ and $f_a:\mathbb{X}\rightarrow \mathbb{R}$ defined as $f_a(x):=|f(x)-a|$, define $\mathbb{X}_r$ as follows:
 \begin{equation}
     \mathbb{X}_r\left(f_a\right)=f_a^{-1}[0, r] 
 \end{equation}
For $h:\mathbb{X}\rightarrow \mathbb{R} $, such that $\|f-h\|_{\infty ,\mathbb{X}}< r$, $h^{-1}(a)$ is included in $\mathbb{X}_r\left(f_a\right)$:
\begin{equation}
   h^{-1}(a) \hookrightarrow \mathbb{X}_r\left(f_a\right),
\end{equation}
and the inclusion induces the homomorphism of the homology:
\begin{equation}
    \mathrm{j}_h: \mathrm{H}_n\left(h^{-1}(a)\right) \rightarrow \mathrm{H}_n\left(\mathbb{X}_r\left(f_a\right)\right)
\end{equation}
In addition, as $f^{-1}(a-r)$ and $f^{-1}(a+r)$ are also included in $\mathbb{X}_r\left(f_a\right)$, the inclusion
\begin{equation}
    \iota^0: f^{-1}(a+r) \hookrightarrow \mathbb{X}_r\left(f_a\right),
\end{equation}
and
\begin{equation}
    \iota^1: f^{-1}(a-r) \hookrightarrow \mathbb{X}_r\left(f_a\right),
\end{equation}
induce the homomorphisms $\iota^0_*$ and $\iota^1_*$ of the homology:
\begin{equation}
   \iota^0_*: \mathrm{H}_n\left(f^{-1}(a+r)\right) \rightarrow \mathrm{H}_n\left(\mathbb{X}_r\left(f_a\right)\right),
\end{equation}
and
\begin{equation}
   \iota^1_*: \mathrm{H}_n\left(f^{-1}(a-r)\right) \rightarrow \mathrm{H}_n\left(\mathbb{X}_r\left(f_a\right)\right).
\end{equation}
\\Define $\mathrm{B}_{0, r}$ and $\mathrm{B}_{1, r}$ as the images of two homomorphisms:
$\mathrm{B}_{0, r}:= \iota^0_*\left(\mathrm{H}_n\left(f^{-1}(a+r)\right)\right) $ and $\mathrm{B}_{1, r}:= \iota^1_*\left(\mathrm{H}_n\left(f^{-1}(a-r)\right)\right) $.
Define $\mathrm{U}_n(r)$ as
\begin{equation}
    \mathrm{U}_n(r)=\bigcap_{\|h-f\|_{\infty, \mathbb{X}} \leq r} \mathrm{im}(j_h).
\end{equation}
Then, the following equation holds:
    \begin{equation}
        \mathrm{U}_n(r)=\mathrm{B}_{0, r} \cap \mathrm{B}_{1, r}.
    \end{equation}
\end{lemma}
We employ the well-known theorem as a lemma.
\begin{lemma}[Hurewicz Theorem (Theorem 2A.1 of \cite{hatcher2000algebraic})]\label{lemma:Hurewicz}
   By regarding loops as singular $1$-cycles, we obtain a homomorphism $h: \pi_1\left(X, x_0\right) \rightarrow \mathrm{H}_1(X)$. If $X$ is path-connected, then $h$ is surjective and has kernel the commutator subgroup of $\pi_1(X)$, so $h$ induces an isomorphism from the abelianization of $\pi_1(X)$ onto $\mathrm{H}_1(X)$.
\end{lemma}

\begin{definition}[Winding Number]
    For a closed curve $c:[0,1] \rightarrow \mathbb{R}^2 - O $, consider $c$ as the element of the fundamental group:
    \begin{equation}
        [c]\in  \pi_1(\mathbb{R}^2 - O, x_0 ) = \mathbb{Z},
    \end{equation}
    where the fundamental group $\pi_1(\mathbb{R}^2 - O, x_0 )$ is generated by the curve $\omega_1 = (\cos (2\pi\theta), \sin(2\pi \theta))$. 
Then, a winding number of $c$ is the natural number $[c]$ as an element of $\pi_1(\mathbb{R}^2 - O, x_0 ) = \mathbb{Z}$.
\end{definition}

\begin{lemma}\label{lemma:winding_number_injective}
    For any closed curve $c: S^1\rightarrow \{(x,y)\in \mathbb{R}^2|  1< x^2+ y^2<2 \}$ in the annulus with a winding number larger than $1$, $c$ is not injective. 
\end{lemma}
\begin{proof}
    Assume that $c$ is an injective curve.
    By the Jordan Curve Theorem (See Proposition 2B.1 of \cite{hatcher2000algebraic} for details), an injective curve bounds a region homeomorphic to the disk.
    Therefore, there exists an embedding $C: D^2 \rightarrow \mathbb{R}^2$ such that its restriction to the boundary is equal to $c$.
\begin{equation}
    C|_{S^1} = c.
\end{equation}
    Because $S^1\hookrightarrow D^2 - \{O\}$ induces an isomorphism of the fundamental group, the degree should be $1$ or $-1$.
    Therefore, a curve with a winding number larger than $1$ is not injective.  
\end{proof}

Using the lemmas, we prove the following theorem.
\begin{theorem}\label{thm:224}
    $w(2,2) = 4$.
\end{theorem}
Proof of Theorem \ref{thm:224} is provided in Appendix \ref{Appendix:proof_thm_224}.
\begin{corollary}
    \begin{equation}
        w_{min}(2,2, \text{ReLU}) = w_{min}(2,2, \text{Leaky-ReLU}) = 4.
    \end{equation}
\end{corollary}
\begin{proof}
    The lower bound $w_{min}(2,2, \text{ReLU})= w(2,2) \geq 4$ is the exact consequence of Theorem \ref{thm:224} and Corollary \ref{cor:lower_upper}.
    The upper bound is by \citet{hanin2017approximating}.
\end{proof}

\section{Conclusion}\label{sec:conclusion}

In this paper, we have introduced a novel upper bound for the minimum width of a deep, narrow MLP necessary for achieving universal approximation within continuous function spaces.
While our derived upper bound demonstrates optimality under specific conditions, we propose that the approach of approximating arbitrary functions through diffeomorphisms could lead to achieving optimality across all cases.
Investigating this perspective presents an intriguing avenue for future research.
Furthermore, we anticipate that analyzing the quantitative approximation capacity of general MLPs from the standpoint of diffeomorphisms may yield valuable insights.

\section*{Acknowldegement}
This work was supported by a KIAS Individual Grant [AP092801] via the Center for AI and Natural Sciences at Korea Institute for Advanced Study. 
\appendix
\section{Proofs}
\subsection{Proof of Lemma \ref{lemma:increasing_with_leaky}}\label{Proof of Lemma_increasing_with_leaky}
\begin{proof}
        Because increasing piecewise linear functions are dense in the space of increasing continuous functions defined on a compact interval, it is sufficient to prove that for an arbitrary natural number $n\in \mathbb{N}$ and an increasing piecewise linear function $f$ with $n$ breakpoints, we have $f\in U_{n}$.
        We will use mathematical induction on $n$. 
        For $n=0$, there is nothing to prove.
        Now, assume that the induction hypothesis is satisfied for some $n=n_0$. 
        Consider the case of $n=n_0+1$, where we have an increasing piecewise linear function $f$ with $n_0+1$ breakpoints, denoted as $\alpha_1 <\alpha_2 <\dots < \alpha_{n_0+1}$.
        The function $f$ is affine on each of the intervals $(-\infty, \alpha_1], [\alpha_1,\alpha_2], \dots, [\alpha_{n_0}, \alpha_{n_0+1}], [\alpha_{n_0+1}, \infty)$.  
        Now, let $f$ has values as follows: 
        \begin{equation}
            f(x)= \begin{cases}
               f(\alpha_{n_0+1}) + \gamma_1 (x - \alpha_{n_0+1})  &\text{ if } x\in [\alpha_{n_0}, \alpha_{n_0+1}]
                \\ f(\alpha_{n_0+1}) + \gamma_2 (x - \alpha_{n_0+1}) &\text{ if } x \in [\alpha_{n_0+1}, \infty)
            \end{cases}.
        \end{equation}
        Consider the function $f_0$ defined as follows:
        \begin{equation}
            f_0(x):= 
            \begin{cases}
                f(x) &\text{ if } x \in (-\infty, \alpha_{n_0+1}]
                 \\ f(\alpha_{n_0+1}) + \gamma_1 (x - \alpha_{n_0+1})&\text{ if } x\in [\alpha_{n_0+1}, \infty)
            \end{cases}.
        \end{equation}
      The function $f_0$ coincides with $f$ on the interval $(-\infty, \alpha_{n_0+1}]$ and is affine on the interval $[\alpha_{n_0}, \infty)$.
        This means that the affine function on the interval $[\alpha_{n_0}, \alpha_{n_0+1}]$ naturally extends to the interval $[\alpha_{n_0}+1, \infty)$ with the same slope.
        Therefore, $f_0$ has $n_0$ breakpoints, and by the induction hypothesis, $f_0\in U_{n_0}$.
        We can express $f$ in terms of $f_0$ as follows:
        \begin{equation}
            f(x) = \frac{\gamma_2}{\gamma_1}\lr{\frac{\gamma_1}{\gamma_2}} \left(f_0(x) - f(\alpha_{n_0+1}) \right)+ f(\alpha_{n_0+1}).
        \end{equation}
        Then, $f\in U_{n_0+1}$, and the induction hypothesis is satisfied for $n=n_0+1$.
        It completes the proof.
    \end{proof}

\subsection{Proof of Lemma \ref{lemma:approximating_acf}}\label{proof_of_lemma_approximating_acf}
\begin{proof}
    For $\beta\in \mathbb{R}_{+}$, $a,c\in \mathbb{R}$ and $b\in \mathbb{R}^{d-1}$, we define the function $g$ as follows:
    \begin{equation}
        g:  (x_1,x_2,\dots, x_d) \mapsto (x_1, x_2, \dots, x_{d-1}, x_d + a\lr{\beta}(b\cdot x_{1:d-1}+c)).
    \end{equation}
    Then, we will prove that $g\prec \narrow{\lr{}}{d,d,d}$.
    If $b$ is the zero vector, $g$ is a constant adding function, and the statement is satisfied.
    If $b$ is not the zero vector, for $b = (b_1, \dots, b_{d-1})$, there exists an index $1\leq i\leq d-1$ such that $b_i\neq 0$.
    We can define an invertible affine transformation $W\in \operatorname{IAff}_{d}$ as follows:
    \begin{equation}
        W: (x_1,x_2,\dots, x_d) \mapsto (x_1, x_2, \dots, x_{i-1}, b\cdot x_{1:d-1}+c, x_{i+1},\dots, x_d).
    \end{equation}
    Because $b_i$ is nonzero, $W$ is invertible.
    Then, by applying $\lr{\beta}$ to the $i$-th component, we get 
\begin{equation}
    (x_1,\dots, x_{i-1}, \lr{\beta}(b\cdot x_{1:d-1}+c),x_{i+1},\dots, x_d) \prec \narrow{\lr{}}{d,d,d}.
\end{equation}
By adding $a$ times the $i$-th component to the last component, we have
\begin{equation}
    (x_1,\dots, \lr{\beta}(b\cdot x_{1:d-1}+c),\dots, x_d + a\lr{\beta}(b\cdot x_{1:d-1})+c) \prec \narrow{\lr{}}{d,d,d}.
\end{equation}
By applying $\lr{\frac{1}{\beta}}$ to the $i$-th component and applying $W^{-1}$, we get
\begin{equation}\label{eq:add}
    (x_1, \dots, x_{d-1}, x_d + a\lr{\beta}(b\cdot x_{1:d-1}+c))\prec \narrow{\lr{}}{d,d,d}.
\end{equation}

Next, we will prove that for the function $h$ defined as
    \begin{equation}\label{eq:add_arbitrary}
        h: (x_1,x_2,\dots, x_d) \mapsto (x_1, x_2, \dots, x_{d-1}, x_d + t(x_1,\dots, x_{d-1})),
    \end{equation}
    $h\prec \narrow{\lr{}}{d,d,d}$.
    By the UAP of two-layered neural networks \citep{leshno1993multilayer}, for arbitrary $\epsilon>0$ and a compact set $K\subset \mathbb{R}^{d-1}$, there exist $\beta\in \mathbb{R_{+}}$, $a_i,c_i\in \mathbb{R}$, and $b_i\in \mathbb{R}^{d-1}$ such that the following inequality holds:
    \begin{equation}
        \left\|t(x_{1:d-1}) - \sum_{i=1}^n a_i\lr{\beta}(b_i\cdot x_{1:d-1}+c_i) \right\|_{\infty, K}<\epsilon.
    \end{equation}
By composing Eq (\ref{eq:add}) for $n$ different $a_i, b_i$, and $c_i$, we obtain
\begin{equation}
    \left(x_1, \dots, x_{d-1}, x_d + \sum_{i=1}^n a_i\lr{\beta}\left(b_i\cdot x_{1:d-1}+c_i\right) \right)\prec\narrow{\lr{}}{d,d,d}.
\end{equation}
Thus, $h\prec \narrow{\lr{}}{d,d,d}$.

Finally, by compositing operations described so far, we demonstrate that any ACF can be approximated by $\narrow{\lr{}}{d,d,d}$.
It is achieved by combining the following four operations:
\begin{itemize}
    \item Apply the logarithm to the last component.
    \item Add $\operatorname{log}(s(x_1,\dots, x_{d-1}))$ to the last component.
    \item Apply the exponential function to the last component.
    \item Add $t(x_1,\dots, x_{d-1})$ to the last component.
\end{itemize}
This results in the following transformation:
\begin{multline}
        (x_1,x_2,\dots, x_d) \mapsto (x_1, x_2, \dots, x_{d-1}, \operatorname{exp}\left(\operatorname{log}(x_d ) + \operatorname{log}(s)\right) + t )
        \\ =(x_1, x_2, \dots, x_{d-1},s x_d + t)        \prec\narrow{\lr{}}{d,d,d}.
\end{multline}
It completes the proof.
\end{proof}

\subsection{Proof of Lemma \ref{lemma:leaky_ReLU_technical}}
 \begin{proof}\label{proof_leaky_ReLU_technial}
We begin by observing that it is sufficient to consider only those functions $b$ that satisfy $b(x_{1:d-1})\geq 1$ for all $x\in K$.
Let's define $\beta$ as $\beta:=\operatorname{inf}_{x\in K}b(x_{1:d-1})$.
We introduce a function $\widetilde{F}(x):= (x_1, \dots, x_{d-1}, \widetilde{f}(x))$ defined as follows:
\begin{equation}
    \widetilde{f}(x):=\beta\lr{\frac{1}{\beta}}\left(f(x))\right).
\end{equation}
If $F\in  \overline{\left.\narrow{\lr{}}{d,d,d}\right|_{K}}$, then $\widetilde{F}\in  \overline{\left.\narrow{\lr{}}{d,d,d}\right|_{K}}$.
The value of $\widetilde{f}(x)$ can be calculated as:
\begin{equation}
    \widetilde{f}(x) = \begin{cases}
        f(x) &\text{ if } x_d\leq \alpha_1 \\
        \beta f(x) &\text{ if } x_d> \alpha_1
    \end{cases}.
\end{equation}
This ratio $\frac{g(x)}{\widetilde{f}(x)} = \frac{b(x_{1:d-1})}{\beta}\geq 1$ for all $x\in K$, and $\widetilde{F}$ also satisfied all the assumptions of the lemma.
Therefore, we only need to consider functions $b$ that satisfy $b\geq 1$.

Next, we will inductively construct a sequence $\left\{G_i = (x_1,\dots, x_{d-1}, g_i(x))\right\}_{i=1}^{\infty} \subset \overline{\left.\narrow{\lr{}}{d,d,d}\right|_{K}}$ that uniformly converges to $G$ when $x_d = \alpha_2$.
We start with $g_0(x):= f(x)$.
Define $b_i:\mathbb{R}^{d-1}\rightarrow\mathbb{R}$ as 
\begin{equation}
    b_i(x_{1:d-1}):= \frac{g_i(x_{1:d-1},\alpha_2)}{f(x_{1:d-1},\alpha_2)},
\end{equation}
for all $x\in K$.
Define $\gamma_i\in \mathbb{R}$ as follows:
\begin{equation}
   \gamma_i:= \operatorname{sup}\left\{\left.\frac{b(x_{1:d-1})}{b_i(x_{1:d-1})}\right| x\in K\right\}.
\end{equation}
Now, define two mutually exclusive sets, $L_{i, 0}$ and $L_{i, 1}$:
\begin{equation}
    L_{i, 0} = \left\{ \left. x_{1:d-1}\in [0,1]^{d-1}\right|1\leq \frac{b(x_{1:d-1})}{b_i(x_{1:d-1})}\leq \gamma_i^{\frac{1}{3}}\right\}.
\end{equation}
\begin{equation}
    L_{i, 1} = \left\{ \left. x_{1:d-1}\in [0,1]^{d-1}\right| \gamma_i^{\frac{2}{3}}\leq \frac{b(x_{1:d-1})}{b_i(x_{1:d-1})}\leq \gamma_i \right\}.
\end{equation}
Define a distance metric $D$ as:
\begin{equation}
    D(x, C): = \operatorname{inf}_{y\in C}\|x-y\|_2,
\end{equation}
and then define $ \phi_i:\mathbb{R}^{d-1}\rightarrow\mathbb{R} $ as:
\begin{equation}
    \phi_i(x):= \frac{D(x,L_{i, 0})}{D(x,L_{i, 0}) +D(x,L_{i, 1})}.
\end{equation}
The function $\phi_i$ satisfies the inequality $0\leq \phi_i(x_{1:d-1})\leq 1$ for all $x\in K$, has value zero on $L_{i,0}$, and has value one on $L_{i,1}$.
Define $h_i:\mathbb{R}^{d-1}\rightarrow\mathbb{R}$ as follows:
\begin{equation}
    h_i(x_{1:d-1}):= (1-\phi_i(x_{1:d-1}))g_i(x_{1:d-1}, \alpha_2)
\end{equation}
Then, $0\leq h_i(x_{1:d-1})\leq g_i(x_{1:d-1}, \alpha_2)$ for all $x\in K$, has a value of zero on $L_{i,1}$ and has a value of $g_i(x_{1:d-1}, \alpha_2)$ on $L_{i,0}$.
\\Now, we define $g_{i+1}(x)\in  \overline{\left.\narrow{\lr{}}{d,d,d}\right|_{K}}$ as follows:
\begin{equation}
    g_{i+1}(x):= \gamma_i^{\frac{1}{3}} \lr{\gamma_i^{-\frac{1}{3}}}(g_i(x) - h_i(x_{1:d-1})) + h_i(x_{1:d-1}).
\end{equation}
We have 
    \begin{equation}
        g_{i+1}(x)  
        \begin{cases}
             = g_i(x) = 0 &\text{ if } x_d\leq \alpha_1 
            \\=g_i(x)  &\text{ if } x_d = \alpha_2 \text{ and } x_{1:d-1}\in L_{i,0}
            \\=\gamma_i^{\frac{1}{3}} g_i(x)   &\text{ if } x_d = \alpha_2 \text{ and } x_{1:d-1}\in L_{i,1}
            \\ \leq \gamma_i^{\frac{1}{3}} g_i(x) &\text{ if } x_d = \alpha_2 \text{ and } x_{1:d-1}\notin L_{i,0}\cup L_{i,1}
        \end{cases}.
    \end{equation}
Thus, for $x$ that $x_d=\alpha_2$ and $x_{1:d-1}\in L_{i,0}$, we have
\begin{equation}
    \frac{g(x)}{g_{i+1}(x)} = \frac{g(x)}{g_{i}(x)} = \frac{b(x_{1:d-1})}{b_i(x_{1:d-1})}.
\end{equation}
As $1 \leq \frac{b(x_{1:d-1})}{b_i(x_{1:d-1})}\leq \gamma_i^{\frac{1}{3}}$ for $x_{1:d-1}\in L_{i,0}$, we get $1 \leq  \frac{g(x)}{g_{i+1}(x)}\leq \gamma_i^{\frac{1}{3}}$.

For $x$ that $x_d=\alpha_2$ and $x_{1:d-1}\in L_{i,1}$,
\begin{equation}
    \frac{g(x)}{g_{i+1}(x)} = \frac{g(x)}{ \gamma_i^{\frac{1}{3}} g_{i}(x)} = \gamma_i^{-\frac{1}{3}}\frac{b(x_{1:d-1})}{b_i(x_{1:d-1})} .
\end{equation}
As $\gamma_i^{\frac{2}{3}} \leq \frac{b(x_{1:d-1})}{b_i(x_{1:d-1})}\leq \gamma_i$ for $x_{1:d-1}\in L_{i,1}$, we get $1\leq \gamma_i^{\frac{1}{3}} \leq  \frac{g(x)}{g_{i+1}(x)}\leq \gamma_i^{\frac{2}{3}}$.

For $x$ that $x_d=\alpha_2$ and $x_{1:d-1}\notin L_{i,0}\cup L_{i,1}$,
\begin{equation}
  \gamma_i^{-\frac{1}{3}} \frac{b(x_{1:d-1})}{b_i(x_{1:d-1})} = \frac{g(x)}{ \gamma_i^{\frac{1}{3}} g_{i}(x)} \leq \frac{g(x)}{g_{i+1}(x)} \leq   \frac{g(x)}{g_{i}(x)} = \frac{b(x_{1:d-1})}{b_i(x_{1:d-1})} .
\end{equation}
As $\gamma_i^{\frac{1}{3}} \leq \frac{b(x_{1:d-1})}{b_i(x_{1:d-1})}\leq \gamma_i^{\frac{2}{3}}$ for $x_{1:d-1}\notin L_{i,0}\cup L_{i,1}$, we get $ 1 \leq  \frac{g(x)}{g_{i+1}(x)}\leq \gamma_i^{\frac{2}{3}}$.

We obtain the following results: for all $x\in K$, where $x_d = \alpha_2$, we have $1 \leq \frac{g(x)}{g_{i+1}(x)} = \frac{b(x_{1:d-1})}{b_{i+1}(x_{1:d-1})}\leq \gamma_i^{\frac{2}{3}}$. This implies $1\leq \gamma_{i+1}\leq \gamma_{i}^{\frac{2}{3}}$.
Consequently, as $i$ tends towards infinity, $\gamma_i$ converges to one.
Therefore, $\frac{g(x_{1:d-1}, \alpha_2)}{g_{i+1}(x_{1:d-1},\alpha_2)}$ uniformly converges to one as $i$ increases, implying that $G_i$ converges to $G$.
As a result, there exists a function$G = (x_1, \dots, x_{d-1},g(x_{1:d}))\in  \overline{\left.\narrow{\lr{}}{d,d,d}\right|_{K}}$ such that $g(x_{1:d-1},\alpha_2) = b(x_{1:d-1})f(x_{1:d-1},\alpha_2)$.

To check that $G$ is a single-coordinate transformation, we can observe that:
\begin{equation}
    g_{i+1}(x_{1:d-1}, x_d) - g_{i+1}(x_{1:d-1}, x'_d) > g_{i}(x_{1:d-1}, x_d) - g_{i}(x_{1:d-1}, x'_d),
\end{equation} 
for $x_d > x'_{d}$, which implies that $g(x_{1:d-1}, x_d) - g(x_{1:d-1}, x'_d)> g_0(x_{1:d-1}, x_d) - g_0(x_{1:d-1}, x'_d)>0$ for all $x\in K$.
Therefore, $g$ satisfies the strictly increasing condition, and $G$ becomes a single-coordinate transformation.

    \end{proof}

\subsection{Proof of Theorem \ref{thm:224}}\label{Appendix:proof_thm_224}
\begin{proof}
    It is obvious that $w(2,2) \leq 4 = 2+2$. 
    Therefore, it is sufficient to prove that $w(2,2)\geq 4$.
    Assume that $w(2,2)\leq 3$.
    Then, for an arbitrary continuous function $f$ in $C([-2,2]^2, \mathbb{R}^2)$, $f$ is contained in $p_{3,2}\circ \overline{\mathrm{Emb}([-2,2]^2, \mathbb{R}^3)} = p_{3,2}\circ \overline{\mathrm{Emb}_{p.l.}([-2,2]^2, \mathbb{R}^3)}$.
Consider a piecewise linear map $f: [-2,2]^2\rightarrow\mathbb{R}^2$ defined as follows:
\begin{equation}
    f(x_1,x_2): = \begin{cases}
         \begin{pmatrix}
             1& -1\\ 0& 2
         \end{pmatrix} \begin{pmatrix}
             x_1 \\ x_2 
         \end{pmatrix} & \text{ if } 0 \leq x_2 \leq  x_1  \\ \\
                  \begin{pmatrix}
             1& -1\\ 2& 0
         \end{pmatrix} \begin{pmatrix}
             x_1 \\ x_2 
         \end{pmatrix} & \text{ if } 0 \leq x_1 \leq  x_2 \\
           - f(x_2, -x_1) & \text{ if } x_1 \leq 0  \text { and }  0 \leq x_2\\
            f( - x_1, -x_2) & \text{ if } x_2 \leq 0
    \end{cases} .
\end{equation}
We can check that $f$ is the piecewise linear double-winding function. 
By the assumption, there exists a piecewise linear embedding $G\in \operatorname{Emb}_{p.l.}([-2,2]^2, \mathbb{R}^3)$ such that 
\begin{equation}
  \| f-  p_{3,2}\circ G\|_{\infty, [-2,2]^2} <\frac{1}{4}.  
\end{equation}

Let $\Sigma :\mathbb{R}^2 \rightarrow \mathbb{R}$ be defined as 
\begin{equation}
    \Sigma: (x_1, x_2)\mapsto |x_1| +|x_2|.
\end{equation}
We can observe that $f$ conserves the level of $\Sigma$: that is, $(\Sigma\circ f)(x) = \Sigma(x)$ for all $x\in \mathbb{R}^2$.
Therefore, 
\begin{equation}
    \left(\Sigma\circ f\right)^{-1}(1)  = \Sigma^{-1}(1) = \left\{\left.(x_1, x_2)\in \mathbb{R}^2  \right| |x_1| +|x_2| =1\right\},
\end{equation}
which is homeomorphic to a circle $S^1$.
Similarly,
\begin{equation}
    \left(\Sigma\circ f\right)^{-1}\left(\frac{1}{2}\right)  = \Sigma^{-1}\left(\frac{1}{2}\right) = \left\{\left.(x_1, x_2)\in \mathbb{R}^2  \right| |x_1| +|x_2| =\frac{1}{2}\right\},
\end{equation}
and 
\begin{equation}
    \left(\Sigma\circ f\right)^{-1}\left(\frac{3}{2}\right)  = \Sigma^{-1}\left(\frac{3}{2}\right) = \left\{\left.(x_1, x_2)\in \mathbb{R}^2  \right| |x_1| +|x_2| =\frac{3}{2}\right\},
\end{equation}
are homeomorphic to $S^1$, and
\begin{equation}
    \left(\Sigma\circ f\right)^{-1}\left( \left[\frac{1}{2}, \frac{3}{2} \right]\right)  = \Sigma^{-1}\left(\left[\frac{1}{2}, \frac{3}{2} \right]\right) = \left\{\left.(x_1, x_2)\in \mathbb{R}^2  \right| \frac{1}{2} \leq |x_1| +|x_2| \leq \frac{3}{2}\right\},
\end{equation}
is homeomorphic to a closed annulus $S^1 \times [0,1]$.
\\Define $g$ as $g:=p_{3,2}\circ G$.
Because $\|f-g\|_{\infty, [-2,2]^2} < \frac{1}{4}$, we have 
\begin{equation}
    |\Sigma \circ f -\Sigma \circ g | < \frac{1}{2}.
\end{equation}

We will apply Lemma \ref{lemma:robustlevelset} to $\Sigma \circ f$.
Because $\left(\Sigma \circ f\right)^{-1}\left(\frac{1}{2}\right) = \Sigma^{-1}(\frac{1}{2})$ and $\left(\Sigma \circ f\right)^{-1}\left(\frac{3}{2}\right)= \Sigma^{-1}(\frac{3}{2})$ are a deformation retract of $\left(\Sigma \circ f\right)^{-1}\left([\frac{1}{2}, \frac{3}{2}]\right)$, the following equation holds:

\begin{multline}
\mathrm{U}_1\left(\frac{1}{2}\right)
 = \mathrm{H}_1\left(\mathrm{B}_{0, \frac{1}{2}}\right) = \mathrm{H}_1\left(\mathrm{B}_{1, \frac{1}{2}}\right) = 
 \\ =    \mathrm{H}_1\left(\left(\Sigma \circ f\right)^{-1}\left(\left[\frac{1}{2}, \frac{3}{2}\right]\right)\right) 
    = \mathrm{H}_1\left(\Sigma ^{-1}\left( 1\right)\right) 
    = \mathbb{Z} 
\end{multline}
Thus, $\mathrm{U}_1(\frac{1}{2}) = \mathbb{Z}$.
Recall that $j_{g}: \mathrm{H}_1\left( \left(\Sigma\circ g\right)^{-1}(1)\right)\rightarrow \mathrm{U}_1(\frac{1}{2})$ is surjective. 

Because $g$ and $\Sigma$ are piecewise linear, $\left(\Sigma\circ g\right)^{-1}(1)$ consists of finite connected components $A_1,\dots, A_k$.
Then, the first homology $\mathrm{H}_1\left( \left(\Sigma\circ g\right)^{-1}(1)\right)$ is decomposed as
\begin{equation}
    \mathrm{H}_1\left( \left(\Sigma\circ g\right)^{-1}(1)\right) = \bigoplus_{i=1}^k \mathrm{H}_1\left( A_i\right),
\end{equation}
and $j_g$ can be decomposed as the sum of homomorphisms $j_{g}^i:\mathrm{H}_1\left( A_i\right)\rightarrow \mathrm{U}_1\left({\frac{1}{2}}\right)$:
\begin{equation}
    j_g(x) = \sum_{i=1}^k j_g^i(x_i),
\end{equation}
for $x = \bigoplus_{i=1}^k x_i$.
As $j_g$ is surjective, we can choose an index $i_0$ such that $j_g^{i_0}$ is a nonzero homomorphism.
Set any basepoint $x_0\in A_{i_0}$.
By Lemma \ref{lemma:Hurewicz}, there exists a surjective Hurewicz homomorphism $h_{1}:  \pi_1\left( A_{i_0}, x_0\right) \rightarrow \mathrm{H}_1\left( A_{i_0}\right)$.
We know that Hurewicz homomorphism $h_2:  \pi_1\left(  \left(\Sigma \circ f\right)^{-1}\left(\left[\frac{1}{2}, \frac{3}{2}\right]\right), x_0\right) \rightarrow \mathrm{H}_1\left( \left(\Sigma \circ f\right)^{-1}\left(\left[\frac{1}{2}, \frac{3}{2}\right]\right)\right)$ is an isomorphism.
By compositing homomorphisms as follows,
\begin{multline}
    \pi_1\left( A_{i_0}, x_0\right) \xrightarrow[]{h_1} \mathrm{H}_1\left( A_{i_0}\right)  
    \\ \xrightarrow[]{j_g^{i_0}} \mathrm{H}_1\left(\left(\Sigma \circ f\right)^{-1}\left(\left[\frac{1}{2}, \frac{3}{2}\right]\right)\right) =\mathrm{U}_1\left(\frac{1}{2}\right) \xrightarrow[]{h_2^{-1}} \pi_1\left(  \left(\Sigma \circ f\right)^{-1}\left(\left[\frac{1}{2}, \frac{3}{2}\right]\right), x_0\right),
\end{multline}
we get the nonzero homomorphism $h_2^{-1}\circ j_g^{i_0} \circ h_1$:
\begin{equation}
    h_2^{-1}\circ j_g^{i_0} \circ h_1 :\pi_1\left( A_{i_0}, x_0\right)\rightarrow \pi_1\left(  \left(\Sigma \circ f\right)^{-1}\left(\left[\frac{1}{2}, \frac{3}{2}\right]\right), x_0\right).
\end{equation}
We can observe that
\begin{equation}
    h_2^{-1}\circ j_g^{i_0} \circ h_1 = \iota_{*},
\end{equation}
where
$\iota_{*}$ is the homomorphism of the fundamental group induced by the inclusion $\iota:  A_{i_0}\hookrightarrow   \left(\Sigma \circ f\right)^{-1}\left(\left[\frac{1}{2}, \frac{3}{2}\right]\right)$.

Now, we will prove that there exists a simple closed curve $\gamma: S^1\rightarrow  \left(\Sigma \circ g\right)^{-1}(1)$, which is homotopic to the cycle $\omega_1: \theta \mapsto (\cos (2\pi \theta), \sin (2\pi \theta))$.
Because $g$ and $\Sigma$ are piecewise linear, $\left(\Sigma\circ g\right)^{-1}(1)$ can be realized by simplicial complex, and we can assume that $\pi_1\left( \left(\Sigma\circ g\right)^{-1}(1), x_0\right)$ is generated by curves consisting of finite segments, and all self-intersection points are breakpoints of curves.
Choose $\gamma_0\in \pi_1\left( \left(\Sigma\circ g\right)^{-1}(1), x_0\right) = \pi_1\left( A_{i_0}, x_0\right)$ such that $\iota_{*}\left(\left[\gamma_0 \right]\right)$ is nonzero.
We will inductively construct a closed curve $\gamma_i$ until $\gamma_i$ has no self-intersection point. 
Assume that $\gamma_i$ has self-intersection point $a \neq b$: $\gamma_{i}(a) = \gamma_{i}(b)$.
Let $\gamma_i^{+}$ be defined as $\gamma_i|_{[a,b]}$ and $\gamma_i^{-}$ be defined as $\gamma_i|_{S^1-(a,b)}$.
Then, $\gamma_i^+$ and $\gamma_i^-$ become two closed curves again with a strictly smaller number of segments than $\gamma_i$.
Because the winding number of $\gamma_i$ is equal to the sum of those of $\gamma_i^+$ and $\gamma_i^-$, at least one of $\gamma_i^+$ or $\gamma_i^-$ has a nonzero winding number, and we set $\gamma_{i+1}$ as the one with a nonzero winding number. 
Each $\gamma_i$ has a strictly smaller number of segments as $i$ increases and has a winding number not equal to zero.
Because $\gamma_0$ has finite segments, this process stops in a finite sequence.
Therefore, we can get a non-self-intersecting curve $\gamma:= \gamma_n$ with a nonzero winding number.
If $\gamma$ has a winding number of which the absolute value is larger than one, by Lemma \ref{lemma:winding_number_injective}, it has a self-intersection point.
Thus, $\gamma$ has a winding number $1$ or $-1$. 
Reverse reparametrization yields a curve with winding number one.

Because $g$ is homotopic to $f$ through the linear interpolation and $\gamma$ is homotopic to $\omega_1$, the compositions of the two functions are homotopic, which implies the same winding number between $g\circ h$ and $f\circ \omega_1$.
Therefore, the winding number of $g\circ \gamma: S^1 \rightarrow S^1 = \Sigma^{-1}(1)$ becomes two. 
Now consider $G$. Because $G$ is an embedding, it is injective.
Therefore, $G|_{\gamma(S^1)}: \gamma(I)\rightarrow S^1\times \mathbb{R}$ is injective.
As the image $G(\gamma(S^1))$ is compact, the image in $S^1 \times \mathbb{R}$ can be embedded in the annuls $\{(x_1, x_2)\in \mathbb{R} | 1- \epsilon \leq |x_1| + |x_2| \leq 1+\epsilon \}$. 
And the map $G\circ \gamma$ has winding number two.
However, by Lemma \ref{lemma:winding_number_injective}, any map with winding number two is not injective, and it becomes a contradiction. 

\end{proof}

\bibliographystyle{elsarticle-harv} 
\bibliography{reference}





\end{document}